\documentclass[10pt,twocolumn,letterpaper]{article}
\usepackage[top=0.7in,bottom=0.7in,left=0.6in,right=0.7in,columnsep=0.35in]{geometry}
\usepackage{hyperref}       
\usepackage{url}            
\usepackage{booktabs}       
\usepackage{amsfonts}       
\usepackage{nicefrac}       
\usepackage{microtype}      
\usepackage{color}
\usepackage{authblk}
\usepackage{blindtext}
\usepackage[nocompress]{cite}
\usepackage{epsfig}
\usepackage{graphicx}
\usepackage{amsmath}
\usepackage{bm}
\usepackage{amssymb}
\usepackage{amsthm}
\usepackage{algorithm}
\usepackage{algorithmic}
\usepackage{color} 
\usepackage{mathdots}
\usepackage{tikz}
\usetikzlibrary{trees}
\usepackage{xfrac}
\usepackage{xr}
\usepackage{caption}
\usepackage{subfigure}
\usepackage{multirow}
\usepackage{authblk}

\newtheorem{remark}{Remark}

\newtheorem{theorem}{Theorem}

\newtheorem{lemma}{Lemma}

\newtheorem{corollary}{Corollary}

\DeclareMathOperator*{\argmin}{arg\,min}

\newcommand{\R}{\mathbb{R}}
\renewcommand{\exp}[1]{e^{#1}}

\newcommand{\dotprod}[2]{\langle #1,#2 \rangle}
\newcommand{\norm}[1]{{\|#1\|}}

\newcommand{\E}{\mathbb{E}}
\newcommand{\Eh}{\widehat{\E}}
\renewcommand{\P}{\mathbb{P}}
\newcommand{\X}{\mathcal{X}}
\newcommand{\HH}{\mathcal{H  }}
\newcommand{\Y}{\mathcal{Y}}
\newcommand{\Z}{\mathcal{Z}}
\newcommand{\D}{\mathcal{D}}
\newcommand{\F}{\mathcal{F}}

\newcommand{\W}{\mathcal{W}}
\newcommand{\Rademacher}{\mathfrak{R}}

\newcommand{\zp}{{\widetilde{z}}}

\newcommand{\vv}{\bm{v}}
\newcommand{\uu}{\bm{u}}
\newcommand{\w}{\bm{w}}
\newcommand{\x}{\bm{x}}
\newcommand{\si}[1]{^{(#1)}}

\newcommand{\hns}{\hspace{-0.025in}}
\newcommand{\eps}{\varepsilon}

\title{On the Statistical Efficiency of Compositional Nonparametric Prediction}

\author[1]{Yixi Xu}
\author[2]{Jean Honorio}
\author[1]{Xiao Wang}
\affil[ ]{xu573@purdue.edu, jhonorio@purdue.edu, wangxiao@purdue.edu }
\affil[1]{Department of Statistics, Purdue University, West Lafayette, IN 47907, USA}
\affil[2]{Department of Computer Science, Purdue University, West Lafayette, IN 47907, USA}
\date{}

\begin{document}
\maketitle
\begin{abstract}
 In this paper, we propose a compositional nonparametric method in which a model is expressed as a labeled binary tree of $2k+1$ nodes, where each node is either a summation, a multiplication, or the application of one of the $q$ basis functions to one of the $p$ covariates. We show that in order to recover a labeled binary tree from a given dataset, the sufficient number of samples is $O(k\log(pq)+\log(k!))$, and the necessary number of samples is  $\Omega(k\log (pq)-\log(k!))$. We further propose a greedy algorithm for regression in order to validate our theoretical findings through synthetic experiments.
\end{abstract}
\section{Introduction}
Nonparametric methods, such as spline-based methods and kernel-based methods, have been widely used in the past 20 years. Most existing methods make assumptions regarding the structure of the model in terms of interactions. For instance, the work of \cite{Ravikumar07} assumes an additive structure of the predictor function, while in \cite{cortes2009learning} the kernel family is defined as polynomial combinations of base kernels of a fixed degree. On the one hand, there is usually insufficient evidence from the data to support the assumption of a specific structure. On the other hand, inclusion of all interactions especially of high order terms would be burdensome for computing especially when the data is high dimensional. A commonly used strategy is to only include low order interactions into the model \cite{cortes2009learning}. However, this would still be a restrictive assumption.

Our goal is to discover the complex structure of the \emph{predictor} function in a concise manner. In contrast, existing methods focus on the discovery of the structure of \emph{kernels} \cite{cortes2009learning,structure}. As an illustrative example for predictor functions, consider the work of Schmidt et al. \cite{schmidt2009distilling}, which discovered physical laws from experimental data, and provided concise analytical expressions that are amenable to human interpretation. 

We build our model by compositionally adding or multiplying basis functions applied to specific dimensions of the covariate. This model is structurally equivalent to a labeled binary tree. The sum-product structure has demonstrated its versatility for several problems. Examples include sum-product networks for computation of partition functions and marginals of high-dimensional distributions \cite{poon2011sum} and structure discovery in nonparametric regression for automatic selection of the kernel family \cite{structure}. 

Our model is a generalization of several popular methods. For illustration, consider the following examples: 
\begin{itemize}
\item Tensor product spline surfaces\cite{spline}: Assume there are two covariates $\x=(x_1,x_2)$,  and define $g(\x) =\sum\limits_{i=1}^{q}\sum\limits_{j=1}^{q}\beta_{ij}\phi_i(x_1)\phi_j(x_2)$, given the basis functions $\phi_1,\dots,\phi_q: \R \to \R$. For simplicity, assume $q=2$, then Figure \ref{treesplinetensor} is one visualization of $g$, where $\beta_{11}=w_1w_3,\beta_{12}=w_1w_4,\beta_{21}=w_2w_3,\beta_{22}=w_2w_4$.
\item Sparse additive models \cite{Ravikumar07}: Assume that $g(\x)$ has an additive decomposition, where $\bm{x}=(x_1,\dots,x_p)$. Define $g(\bm{x})=\sum\limits_{j=1}^p\phi_{a_j}(x_j)$, where $a_1,\dots,a_p \in \{1,\dots,q\}$ and such that $\sum\limits_{j=1}^p\mathbb{I}(\phi_{a_j}\neq 0)\le s$ for some integer $s \ll p$.  

\item Tensor decomposition: Given a set of $q$ functions $\phi_1,\dots,\phi_q$ and a tensor $y_{ijk}$ for $i,j,k=1,\dots,p$. The problem is to find the indices $a_r,b_r,c_r \in \{ 1,\dots, q\}$ for $r=1,\dots,R$, that minimize:
$\sum\limits_{i=1}^p\sum\limits_{j=1}^p\sum\limits_{k=1}^p \left( \sum\limits_{r=1}^R { w_r \phi_{a_r}(i) \phi_{b_r}(j) \phi_{c_r}(k) } - y_{ijk} \right) ^2$.
Note that $\sum\limits_{r=1}^R { w_r \phi_{a_r}(i) \phi_{b_r}(j) \phi_{c_r}(k) }$ can be written as a fixed weighted labeled binary tree. Figure \ref{treetensor} illustrates the case when $R=2$.
\end{itemize}
Our contribution is as follows. First, we propose a general compositional sum-product nonparametric method, in which a model is expressed as a weighted labeled binary tree. Second, we provide a generalization bound that holds for any data distribution and any weighted labeled binary tree. We show that $O(k \log (pq)+\log k!)$ samples are sufficient, by using Rademacher-complexity arguments.
Third, we further show that $\Omega(k \log (pq)-\log k!)$ samples are necessary, by using information-theoretic arguments. Thus, our sample complexity bounds are tight. Furthermore, since the sample complexity is \emph{logarithmic} in $p$ and $q$, our method is statistically suitable for high dimensions and a large number of basis functions. Finally, we propose a well-motivated greedy algorithm for regression in order to validate our theoretical findings.

For comparison with results on sparse additive models, the work of  \cite {Ravikumar07} presents an $L_1$-regularization approach. Additionally, a sample complexity of $O(q\log((p-s)q))$ was shown to be sufficient for the correct identification of the basis functions in the sparse additive model. Note that in our work, we are interested in generalization bounds for the  prediction error. The necessary number of samples for sparse additive models was analyzed in \cite{Raskutti09}, where a sample complexity of $\Omega(s\log p)$ was found for the recovery of a function that is close to the true function in $L_2$-norm. Our sample complexity guarantee of $O(k \log p)$ matches this bound.

The paper is structured as follows. In Section 2, we provide a generalization bound. Section 3 discusses the necessary number of samples. In Section 4, we propose a greedy search algorithm for regression. In Section 5, we validate our theoretical results through synthetic experiments.

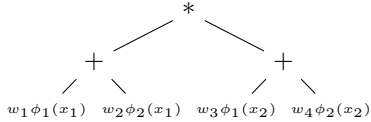
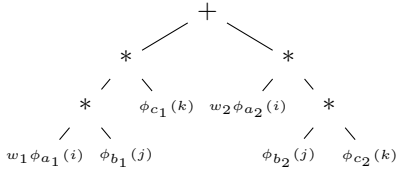
\begin{figure}
  \centering
  \subfigure[Tensor product spline surfaces.]{
     \begin{tikzpicture}
[scale=.36,
level distance=1.8cm,
  level 1/.style={sibling distance=7cm},
  level 2/.style={sibling distance=3.5cm},
  empty/.style = {circle, draw, fill = white, minimum size = .2cm}]
  \node {*}
    child {node {+}
      child {node {\tiny  $w_1\phi_1(x_1)$}}
      child {node {\tiny  $w_2\phi_2(x_1)$}}
    }
    child {node {+}
    child {node {\tiny  $w_3\phi_1(x_2)$}}
      child {node {\tiny  $w_4\phi_2(x_2)$}}
    };
\end{tikzpicture}
\label{treesplinetensor}}
      \hspace{0.01\linewidth}
  \subfigure[Tensor decomposition.]{
\begin{tikzpicture}
[scale=.36,
level distance=1.8cm,
  level 1/.style={sibling distance=6cm},
  level 2/.style={sibling distance=3cm},
  empty/.style = {circle, draw, fill = white, minimum size = .2cm}]
  \node {+}
    child {node {*}
      child {node {*}
      child {node {\tiny  $w_1\phi _{a_1}(i)$}}
      child {node {\tiny  $\phi_{b_1}(j)$}}}
      child {node {\tiny  $\phi_{c_1}(k)$}}
    }
    child {node {*}
    child {node {\tiny  $w_2\phi_{a_2}(i)$}}
      child {node {*}
      child {node {\tiny  $\phi_{b_2}(j)$}}
      child {node {\tiny  $\phi_{c_2}(k)$}}}
    };
\end{tikzpicture}
\label{treetensor}}
  \caption{Examples of tensor product spline surfaces and tensor decomposition.}
\end{figure}
\section{Compositional Nonparametric Trees for the General Prediction Problem}
In this section, we define the general prediction problem, and then propose a solution via a compositional nonparametric method, in which a model is defined as a weighted labeled binary tree. In this tree, each node represents a multiplication, an addition, or the application of a basis function to a particular covariate.
\paragraph{The General Prediction Problem.}
Assume that $\x_1,\dots,\x_n$ are $n$ independent random variables on $\X=\mathbb{R}^p$, $y_1,\dots,y_n$ are on $\mathcal{Y}\subseteq\mathbb{R}$. The general prediction problem is defined as
\begin{equation}
\label{general}
y_i=t(g(\x_i)+\epsilon_i),
\end{equation}
where $t: \mathbb{R} \rightarrow \mathcal{Y}$ is a fixed function related to the prediction problem, $g: \R^p \to \R$ is an unknown function, and $\epsilon_i$ is an independent noise. We provide two examples in order to illustrate how to adopt equation (\ref{general}) to different settings. For regression, we define $t(z)=z$, while for classification, we define $t(z)=sign(z)$.

\paragraph{The Labeled Binary Tree.}
\begin{figure}  
 \centering
  \subfigure[A labeled binary tree.]{
   \centering
  \begin{tikzpicture}
[scale=.36,
level distance=1.8cm,
  level 1/.style={sibling distance=6cm},
  level 2/.style={sibling distance=3cm},
  empty/.style = {circle, draw, fill = white, minimum size = .2cm}]
  \node {*}
    child {node {+}
      child {node {\tiny  $\phi_1(x_2)$}}
      child {node {\tiny  $\phi_3(x_1)$}}
    }
    child {node {+}
    child {node {\tiny  $\phi_3(x_2)$}}
      child {node {\tiny  $\phi_1(x_3)$}}
    };
\end{tikzpicture}
\label{tree1}
}
      \hspace{0.01\linewidth}
  \subfigure[A weighted labeled binary tree.]{
\centering
               \begin{tikzpicture}
[scale=.36,
level distance=1.8cm,
  level 1/.style={sibling distance=7cm},
  level 2/.style={sibling distance=3.5cm},
  empty/.style = {circle, draw, fill = white, minimum size = .2cm}]
  \node {*}
    child {node {+}
      child {node {\tiny  $w_1\phi_1(x_2)$}}
      child {node {\tiny  $w_2\phi_3(x_1)$}}
    }
    child {node {+}
    child {node {\tiny  $w_3\phi_3(x_2)$}}
      child {node {\tiny  $w_4\phi_1(x_3)$}}
    };
\end{tikzpicture}
\label{wtree}}
\caption{Two tree examples. }
\end{figure}
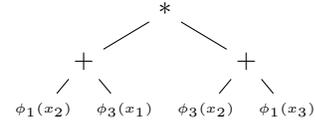
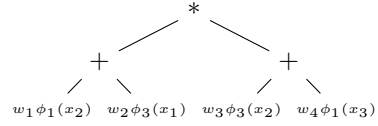

We define a functional structure built compositionally by adding and multiplying a
small number of basis functions. A straightforward visualization of this structure is a labeled binary tree. Given an infinite set of basis functions $\Phi=\{ \phi_l, l=1,2,\cdots,\infty\}$ on $\mathbb{R} \rightarrow [-1,1]$ and a truncation parameter $q$, $\mathcal{F}_{2k+1}$ is a set of binary trees where:
\begin{enumerate}
\item there are no more than $2k+1$ nodes,
\item the labels of non-leaf nodes can be either ``+'' or ``*'',
\item the label of a leaf node can only be a function in $\Phi$ on a specific dimension of the covariate $\bm{x}=(x_1,\dots,x_p)$, that is $\phi_i(x_j)$ for any $i=1,\dots,q$ and $j=1,\dots,p$,
\end{enumerate}
Figure \ref{tree1} gives an example of a labeled binary tree with seven nodes. All the leaves are $\phi_i(x_j)$s, while all non-leaf nodes are operations. Note that if we switch the left sub-tree and the right sub-tree, we obtain an equivalent structure.

As pointed out later in Remark \ref{remark1}, in the nonparametric setting, both $k$ and $q$ are allowed to grow as a function of $n$.
\paragraph{The Weighted Labeled Binary Tree.}
It is easy to show that a labeled binary tree with $2k+1$ nodes has the following properties:
\begin{enumerate}
\item It includes $k$ operations.
\item It has $k+1$ leaves.
\end{enumerate}
An easy way to add weights is to directly add weights to each leaf node, as shown in Figure \ref{wtree}. So given a tree structure $f\in \mathcal{F}_{2k+1}$, we can define $\mathcal{W}(f)$ as the set of all weighted labeled binary trees given $f$, with constraint $\norm{\w}_1\le 1$. Additionally, we define 
\begin{equation}
\mathcal{W}_{2k+1}=\bigcup\limits_{f\in \mathcal{F}_{2k+1}}\mathcal{W}(f).
\end{equation}

For a fixed  $f \in \F_{2k+1}$, any $h \in \mathcal{W}(f)$ can be rewritten as a summation of some basis functions and some  productions of basis functions. For instance, given $\w$ and the labeled binary tree structure $f_0$ in Figure \ref{tree1}, Figure \ref{wtree} represents a function $h(x;f_0,\w)=(w_1\phi_1(x_2)+w_2\phi_3(x_1))*(w_3\phi_3(x_2)+w_4\phi_1(x_3))$, and it is the summation of 4 interactions $w_1w_3\phi_1(x_2)\phi_3(x_2)$, $w_1w_4\phi_1(x_2)\phi_1(x_3)$, $w_2w_3\phi_3(x_1))\phi_3(x_2)$, and $w_2w_4\phi_3(x_1)\phi_1(x_3)$. Equivalently, 
$h(x;f_0,\w)=\langle \vv,\uu \rangle$, where $\vv=\psi_{f_0}^{v}(\w)=(w_1w_3,w_1w_4,w_2w_3,w_2w_4)$ and 
$\uu=\psi_{f_0}^{u}(\x)=(\phi_1(x_2)\phi_3(x_2),\phi_1(x_2)\phi_1(x_3),\phi_3(x_1)\phi_3(x_2),\phi_3(x_1)\phi_1(x_3))$.
Similarly, for any labeled binary tree $f$, we could write $h=h(x;f,\w)\in \mathcal{W}(f)$ as an inner product of two vectors $\vv$ and $\uu$:
\begin{align}
\label{summationdecompsition}
&h(x;f,\w)=\langle \vv,\uu \rangle,\quad \vv=\psi_{f}^{v}(\w), \quad \uu=\psi_{f}^{u}(\x),
\end{align}
where the transformation function $\psi_{f}^{v}$ and $\psi_{f}^{u}$ depend on $f$. Define the length of the vector $\vv$ and $\uu$ as $M_f$, and $M_f$ also depends on $f$. Define 
\begin{equation}
\label{defM}
M_{2k+1}=\max\limits_{f\in \F_{2k+1}}M_f.
\end{equation}
	\begin{lemma}
		\label{lemma0}
		If $\norm{\w}_1\le 1$ and $\norm{\phi_i}_{\infty}\le 1 \; \forall i$, regardless of $f$, we always have $\norm{\vv}_1 \le 1$
		and $\norm{\uu}_{\infty}\le 1$. 
\end{lemma}
\begin{proof}[Proof sketch]
	By induction.
\end{proof}
(Detailed proofs can be found on Appendix A.)

\section{Sufficient Number of Samples}
In this section, we provide a generalization bound that holds for any data distribution and any labeled binary tree. This not only implies the sufficient number of samples to recover a labeled binary tree from a given dataset, but also guarantees that the empirical risk (i.e., the risk with respect to a training set) is a consistent estimator of the true risk (i.e., the risk with respect to the data distribution). We first bound the size of $\mathcal{F}_{2k+1}$, and then show a Rademacher-based uniform convergence guarantee. 

\paragraph{Properties of the Labeled Binary Tree Set.}
Let $|\mathcal{F}_{2k+1}|$ denote the size of $\mathcal{F}_{2k+1}$: the labeled binary tree set with no more than $2k+1$ nodes. The lemma below gives the upper bound of the size of the functional space, which will be used later to show the uniform convergence.
\begin{lemma}
\label{lemma1}
For $k \geq 1$, we have $|\mathcal{F}_{2k+1}|\le 4k(k)!(pq)^{k+1}$.
\end{lemma}
\begin{proof}[Proof sketch]
By induction.
\end{proof}
The lemma below gives the upper bound of $M_{2k+1}$, which is used to later to bound the Rademacher complexity. Remind that $M_{2k+1}$ is defined in eq.(\ref{defM}).
\begin{lemma}
\label{lemmadecomposition}
$M_{2k+1}<(1.45)^{k+1}$.
\end{lemma}
\begin{proof}[Proof sketch]
By induction.
\end{proof}
\paragraph{Rademacher-based Uniform Convergence.}

Next, we present our first main theorem, which guarantees a uniform convergence of the empirical risk to the true risk, regardless of the tree structure and weights.

Assume that $d : \Y \times \Y \to [0,1]$ is a 1-Lipschitz function related to the prediction problem. For regression, we assume $\mathcal{Y}=\R$, and $d(y,y') = \min(1, (y-y')^2/2)$, while for classification, we assume $\mathcal{Y}=\{-1,1\}$, and $d(y,y^{'})= \min (1,\max (0,1-yy^{'}))$. Let $z=(x,y)\in \Z$, where $\Z=\X\times \Y$. Furthermore, let $\mathcal{H}(f)=\{ h(z)= d(y,g(x)), g\in\mathcal{W}(f)\}$ for a fixed labeled binary tree $f$. Let $\mathcal{H}_{2k+1}$ be a hypothesis class satisfying 
\begin{equation*}
\mathcal{H}_{2k+1}=\bigcup\limits_{f\in \mathcal{F}_{2k+1}}\mathcal{H}(f).
\end{equation*}

For every $h \in \mathcal{H}(f)$, we define the true and empirical risks as 
\begin{equation}
\E_\D[h]=\E_{z\sim \D}[h(z)],\quad \Eh_S[h]=\frac{1}{n}\sum\limits_{i=1}^nh(z_i).
\end{equation}

Next, we state our generalization bound that shows that $O(k \log (pq) + \log k!)$ samples are sufficient for learning.
\begin{theorem} \label{thm:radunifconv}
Let $z=(x,y)$ be a random variable of support $\Z$ and distribution $\D$.
Let ${S = \{z_1 \dots z_n\}}$ be a dataset of $n$ i.i.d. samples drawn from $\D$.
Fix ${\delta \in (0,1)}$.
With probability at least ${1-\delta}$ over the choice of $S$, we have:
\begin{align*}
(\forall f\in &\F_{2k+1},\forall h \in \mathcal{H}(f))\\
\E_\D[h] &\leq \Eh_S[h] + 2\sqrt{\frac{k+1}{n}}+ \\
&\sqrt{\frac{(k+1)\log pq+\log 8k(k)!+\log{(1/\delta)}}{2n}}
\end{align*}
\end{theorem}
\begin{proof}
Given a function ${h : \Z^n \to \R}$, we define 
$\E_S[h(S)] = \E_{S \sim \D^n}[h (S)]$.
The function $\varphi_f(S)=\sup_{h \in \HH(f)}{\left( \E_\D[h] - \Eh_S[h] \right)}$ fulfills the condition in McDiarmid's inequality and $\mathcal{H}(f) \subseteq \{ h | h : \Z \to [0,1]\}$, by Lemma \ref{lem:varphilipschitz} (Please see Appendix B.), therefore 
$\P[\varphi_f(S) - \E_S[\varphi_f(S)] \geq \eps]\leq \exp{\frac{-2 \eps^2}{\sum_{i=1}^n{(1/n)^2}}} = \exp{-2 n \eps^2}$. Furthermore, by applying the union bound for all $f \in \mathcal{F}_{2k+1}$, by Lemma \ref{lemma1}, and by Hoeffding's inequality, we have:
\begin{equation*}
\begin{split}
&\mathbb{P}[(\exists f \in \mathcal{F}_{2k+1}), \varphi_f(S) - \E_S[\varphi_f(S)] \geq \eps]]\le \\
&\sum\limits_{f\in \mathcal{F}_{2k+1}}\P[\varphi_f(S) - \E_S[\varphi_f(S)] \geq \eps]\le 2|\mathcal{F}_{2k+1}|e^{-2n\epsilon^2}\\
&\leq 8k(k)!(pq)^{k+1}e^{-2n\epsilon^2}
\end{split}
\end{equation*}
Equivalently,
$
\mathbb{P}[(\forall f \in \mathcal{F}_{2k+1}), \varphi_f(S) - \E_S[\varphi_f(S)] \leq \eps]]\geq 1-8k(k)!(pq)^{k+1}e^{-2n\epsilon^2}$.

Setting ${8k(k)!(pq)^{k+1}e^{-2n\epsilon^2} = \delta}$, we get ${\eps = \sqrt{\frac{(k+1)\log pq+\log 8k(k)!+\log{(1/\delta)}}{2n}}}$.
Thus:
\begin{align}
\label{eq:step1}
& \P\hns \; [(\forall f \in \mathcal{F}_{2k+1}), \varphi_f(S) < \E_S[\varphi_f(S)] + \nonumber \\
& \sqrt{\frac{(k+1)\log pq+\log 8k(k)!+\log{(1/\delta)}}{2n}}] \nonumber \\
 & \geq 1-\delta
\end{align}
Note that by the definition of the supremum, by the definition of the function ${\varphi_f : \Z^n \to \R}$, and by eq.\eqref{eq:step1}, with probability at least ${1-\delta}$, simultaneously for all $f \in \F_{2k+1}$ and $h \in \mathcal{H}(f) $
\begin{align} \label{eq:step2}
\E_\D[h] - \Eh_S[h] & \leq \sup_{h \in \mathcal{H}(f)}{\left( \E_\D[h] - \Eh_S[h] \right)} \nonumber \\
 &\hspace{-.2in} = \varphi_f(S) \nonumber \\
 &\hspace{-.2in}<\sqrt{\frac{(k+1)\log pq+\log 8k(k)!+\log{(1/\delta)}}{2n}}+ \nonumber \\
  &\hspace{-.2in} \E_S[\varphi_f(S)] 
\end{align}
The next step is to bound ${\E_S[\varphi_f(S)]}$ in eq.\eqref{eq:step2} in terms of the Rademacher complexity of $\mathcal{W}(f)$. By the definition of $\varphi_f$, by the ghost sample technique, the Ledoux-Talagrand Contraction Lemma, we can show that
\begin{equation} \refstepcounter{equation}
\E_S[\varphi_f(S)]  = 2\Rademacher_n(\mathcal{H}(f))\le 2\Rademacher_n(\mathcal{W}(f)) \nonumber
\end{equation}
The final step is to bound $\Rademacher_n(\mathcal{W}(f))$, and it is sufficient to bound $\hat{\Rademacher}_S(\mathcal{W}(f))$ for any $f \in \F_{2k+1}$. Then for a fixed  $f \in \F_{2k+1}$, any $g \in \mathcal{W}(f)$ can be rewritten as a summation of no more than $[(1.45)^{k+1}]$ productions of basis functions, where $[m]$ denotes that largest integer smaller than or equal to $m$ according to Lemma \ref{lemmadecomposition}. We could decompose $h=h(x;f,\w)$ as in equation (\ref{summationdecompsition}), thus $h=h(\x;f,\w)=\langle \vv,\uu \rangle$, where $||\vv||_1\le 1$ and $||\uu||_{\infty}\le 1$ by Lemma \ref{lemma0}. By using a technique similar to \cite{Kakade08} for linear prediction, we have
\begingroup
\allowdisplaybreaks
\begin{align*} \refstepcounter{equation}
 \hspace{-0.2in} \hat{\Rademacher}_S(\mathcal{W}(f))& = \E_\sigma\left[ \sup_{g \in \mathcal{W}(f)}{\left( \frac{1}{n} \sum_{i=1}^n{\sigma_i g(\x\si{i})} \right)} \right] \nonumber \\
\hspace{-0.2in}&  = \E_\sigma\left[ \sup_{\norm{\w}_1 \leq 1}{\left( \frac{1}{n} \sum_{i=1}^n{\sigma_i g(\x\si{i};\w,f)} \right)} \right] \nonumber \\
\hspace{-0.2in} &  \le \frac{1}{n} {\rm\ } \E_\sigma\left[ \sup_{\norm{\vv}_1 \leq 1}{\left( \sum_{i=1}^n{\sigma_i \dotprod{\vv}{\uu\si{i}})} \right)} \right] \\
  \hspace{-0.2in}& = \frac{1}{n} {\rm\ } \E_\sigma\left[ \sup_{\norm{\vv}_1 \leq 1}{\dotprod{\vv}{\textstyle{ \sum_{i=1}^n{\sigma_i \uu\si{i}} }}} \right] \nonumber \\ 
 \hspace{-0.2in}&  = \frac{\norm{\vv}_1}{n} {\rm\ } \E_\sigma\left[ \norm{\textstyle{\sum_{i=1}^n{\sigma_i \uu\si{i}}}}_{\infty} \right] \nonumber \\ 
\hspace{-0.2in} & = \frac{1}{n} {\rm\ } \E_\sigma\left[ \sup\limits_j\textstyle{\sum_{i=1}^n{\sigma_i [\uu\si{i}]_j}} \right] \nonumber \\ 
 \hspace{-0.2in}& = \frac{\sqrt{2\log M_{2k+1}}}{n} {\rm\ }  \sup\limits_j\sqrt{\textstyle{\sum_{i=1}^n{ [\uu\si{i}]^2_j}}} \\ 
  \hspace{-0.2in}& \le \frac{\sqrt{2\log M_{2k+1}}}{n} {\rm\ }  \sqrt{\textstyle{n\norm{\uu}^2_{\infty}}} \nonumber\\
  \hspace{-0.2in}&  \le \sqrt{\frac{2\log M_{2k+1}}{n}} {\rm\ }  \nonumber\\
 \hspace{-0.2in} &  \le \sqrt{\frac{2(k+1)\log 1.45}{n}} {\rm\ }  \nonumber\\
 \hspace{-0.2in} & < \sqrt{\frac{k+1}{n}} {\rm\ }  \nonumber
\end{align*}
\endgroup
Finally, we have
$\Rademacher_n(\mathcal{W}(f))=\E_{S\sim \D^n}[\hat{\Rademacher}_S(\mathcal{W}(f))]<\sqrt{\frac{k+1}{n}} $
\qedhere
\end{proof}
\begin{corollary} 
\label{cor1}
Define $\hat{h}=\argmin\limits_{h \in \HH_{2k+1}} \Eh_S[h]$, and $\bar{h}=\argmin\limits_{h \in \HH_{2k+1}} \E_\D[h]$. Then under the same setting of Theorem \ref{thm:radunifconv}, fix $\delta, \epsilon \in (0,1)$,
if \[n\ge \frac{3(k+1)(\log pq +8)+3\log 8k(k)!+6\log{(2/\delta)}}{2\epsilon^2}\], then $\E_\D[\hat{h}]-\E_\D[\bar{h}] \leq  \epsilon$ with probability at least ${1-\delta}$ over the choice of $S$.
\end{corollary}
\begin{proof}
	By Theorem \ref{thm:radunifconv}, with probability at least $1-\delta/2$ over the choice of $S$,
	\begin{align*}
	\E_\D[\hat{h}] &\leq \Eh_S[\hat{h}] + 2\sqrt{\frac{k+1}{n}}+ \\
	&\sqrt{\frac{(k+1)\log pq+\log 8k(k)!+\log{(2/\delta)}}{2n}}
	\end{align*}
By Hoeffding's inequality, with probability at least $1-\delta/2$ over the choice of $S$,
	\begin{align*}
 \Eh_S[\bar{h}] -\E_\D[\bar{h}] &\leq  \sqrt{\frac{\log (2/\delta)}{2n}}
\end{align*}
	Since $\hat{h}$ minimizes $\Eh_S[h]$, $\Eh_S[\hat{h}]\le \hat{\E}_S[\bar{h}]$.  With probability at least $1-\delta$ over the choice of $S$,
	\begin{align*}
	&\E_\D[\hat{h}]-\E_\D[\bar{h}] =\E_\D[\hat{h}]-\Eh_S[\bar{h}]+\Eh_S[\bar{h}]-\E_\D[\bar{h}]\\
	&\le  \E_\D[\hat{h}]-\Eh_S[\hat{h}]+\Eh_S[\bar{h}]-\E_\D[\bar{h}]\\
	&\le  \sqrt{\frac{(k+1)\log pq+\log 8k(k)!+\log{(2/\delta)}}{2n}}\\
	&+2\sqrt{\frac{k+1}{n}}+ \sqrt{\frac{\log (2/\delta)}{2n}}\\
	&\le \sqrt{\frac{3(k+1)(\log pq +8)+3\log 8k(k)!+6\log{(2/\delta)}}{2n}}
	\end{align*}
	Set $ \sqrt{\frac{3(k+1)(\log pq +8)+3\log 8k(k)!+6\log{(2/\delta)}}{2n}} \le \epsilon$. Equivalently $n\ge \frac{3(k+1)(\log pq +8)+3\log 8k(k)!+6\log{(2/\delta)}}{2\epsilon^2}$.
	Note that the last step is due to \[\sqrt{x}+\sqrt{y}+\sqrt{z}\le \sqrt{3x+3y+3z}\] Next, we present a useful remark in the nonparametric setting, where both $k$ and $q$ are allowed to grow as a function of $n$.
	\end{proof}

\begin{remark} 
	\label{remark1}
If $k \in O(\min(n^{1/2-\epsilon},\frac{n^{1-2\epsilon}}{\log p}))$, $q\in O(e^{n^{1/2-\epsilon}})$ for any $\epsilon\in (0,1/2)$, then the generalization error in Theorem \ref{thm:radunifconv} could be uniformly bounded by $O(n^{-\epsilon})$. 
\end{remark}
\section{Necessary Number of Samples}
In this section, we analyze the necessary number of samples to recover a labeled binary tree from a given dataset. To show the necessary number of samples, we restrict the operation to multiplications only, and consider unit weights. Note that the necessary number of samples in restricted ensembles yields a lower bound for the original problem. The use of restricted ensembles is customary for information-theoretic lower bounds \cite{santhanam2012information,Wang10}. We  utilize Fano's inequality as the main proof technique.

We construct a restricted ensemble as follows. Define a sequence of basis functions $\phi_i(z)=\sqrt{2}\cos(i\pi z)$, where $z \in [-1,1]$ for $i=1,\dots,q$. Furthermore, let $\x_i \sim Unif[-1,1]^p$, $\epsilon_i\sim N(0,\sigma_\epsilon^2)$.  Let $S=\{(\x_i,z_i):z_i=g(\x_i)+\epsilon_i,i=1,\dots,n\}$, and  $S^{'}=\{(\x_i,y_i):y_i=t(z_i),i=1,\dots,n\}$, where $t : \R \to \Y$  is a fixed function related to the prediction problem, as introduced in Section 2. This defines a Markov chain $g \to S \to S' \to \hat{g}$. To apply Fano's inequality, we need to further bound the mutual information $\mathbb{I}(g,S^{'})$ by a sum of Kullback-Leibler (KL) divergences of the form $KL(P_{\x ,y|g_i}|P_{\x ,y|g_i'})$ where $g_i$ and $g_i'$ are two different compositional trees. Consider a labeled binary tree subspace $\mathcal{G}_{2k+1}$ of $\mathcal{F}_{2k+1}$, where we only allow for multiplication nodes (i.e., additions are not allowed) and where each covariate $x_j$ of the independent variable $\bm{x}$ is used only once. Furthermore, we consider a restricted ensemble with unit weights. Equivalently, 
\begin{equation*}
\begin{split}
\mathcal{G}_{2k+1}=\{ &g_\mathcal{A}(\bm{x})=\prod\limits_{(i,j) \in \mathcal{A}}\phi_i(x_j):\mathcal{A}\subseteq\{ 1,\dots,q\} \times \{ 1,\dots,p\}, \\
&|\mathcal{A}|\le k+1, \forall (i,j) \in \mathcal{A}, \ l\ne i \Rightarrow (l,j)\not\in \mathcal{A} \}.
\end{split}
\end{equation*}
Let $c=|\mathcal{G}_{2k+1}|=\sum\limits_{i=1}^{k}q^{i+1}{p \choose i+1} $. 

Next, we state our information-theoretic lower bound that shows that $\Omega(k \log (pq) - \log k!)$ samples are necessary for learning.
\begin{theorem}
\label{thm:nec}
Assume nature uniformly picks a true hypothesis $\bar{g}$ from $\mathcal{G}_{2k+1}$. For any estimator $\hat{g}$, if $n\le (\log (q^{k+1}{p \choose k+1})-2\log 2) \sigma^2_\epsilon/2$, then $\P[\hat{g} \neq \bar{g}]\ge \frac{1}{2}$.
\end{theorem}
\begin{proof}
Any $g_\mathcal{A} \in \mathcal{G}_{2k+1}$ can be decomposed by the dimension of $x$: \[g_\mathcal{A}(\x)=\prod\limits_{j=1}^{p}g_j^\mathcal{A}(x_j),\] where $g_j^\mathcal{A} =\phi_{i_j}$ if $\exists (i_j,j) \in \mathcal{A}$, and  $g_j^\mathcal{A} \equiv 1$ if $(i,j) \notin \mathcal{A}$ for any $i$. In addition, $\int_{-1}^1\frac{1}{2}\phi_i(x)dx=0$ and $\langle \phi_i,\phi_{i^{'}}\rangle =\int_{-1}^1\frac{1}{2}\phi_i(x)\phi_{i^{'}}(x)dx=I(i=i^{'})$. Thus, 
\begin{equation*}
\begin{split}
\langle g_\mathcal{A},g_{\mathcal{A}^{'}}\rangle &=\int_{-1}^1\cdots\int_{-1}^1\frac{1}{2^p}g_j^{\mathcal{A}}(x_j)g_j^{\mathcal{A}^{'}}(x_j)dx_1\cdots dx_p \\
&=\prod\limits_{j=1}^p \int_{-1}^1\frac{1}{2}g_j^{\mathcal{A}}(x_j)g_j^{\mathcal{A}^{'}}(x_j)dx_j\\
&=\prod\limits_{j=1}^pI(g_j^{\mathcal{A}}=g_j^{\mathcal{A}^{'}})\\
&=I( g_\mathcal{A}=g_{\mathcal{A}^{'}})
\end{split}
\end{equation*}
	Furthermore,
	\begin{equation}
	\begin{split}
	||g_{\mathcal{A}}-g_{\mathcal{A}^{'}}||^2&=\langle g_{\mathcal{A}},g_{\mathcal{A}}\rangle + \langle g_{\mathcal{A}^{'}},g_{\mathcal{A}^{'}}\rangle -2\langle g_{\mathcal{A}}, g_{\mathcal{A}^{'}} \rangle\\
	&=2I( g_\mathcal{A}=g_{\mathcal{A}^{'}}) 
	\end{split}
	\label{eqn:innerproduct}
	\end{equation}

By the data processing inequality \cite{Cover06} in the Markov chain $g \to S \to S' \to g$, and since the mutual information can be bounded by a pairwise KL bound \cite{Yu97}, we have
\begingroup
\allowdisplaybreaks
\begin{align*}
\mathbb{I}(\bar{g},S^{'})&\le \mathbb{I}(\bar{g},S)\\
&\le \frac{1}{c^2}\sum_{\mathcal{A}}\sum_{\mathcal{A}^{'}}KL(P_{S|g_\mathcal{A}}|P_{S|g_{\mathcal{A}^{'}}})\\
&=\frac{n}{c^2}\sum_{\mathcal{A}}\sum_{\mathcal{A}^{'}}KL(P_{\x,y|g_\mathcal{A}}|P_{\x,y|g_{\mathcal{A}^{'}}})\\
&=\frac{n}{c^2}\sum_{\mathcal{A}}\sum_{\mathcal{A}^{'}}KL(\mathcal{N}(g_{\mathcal{A}},\sigma_\epsilon^2)|\mathcal{N}(g_{\mathcal{A}^{'}},\sigma_\epsilon^2))\\
& =\frac{n}{c^2}\sum_{\mathcal{A}}\sum_{\mathcal{A}^{'}}\ \frac{||g_{\mathcal{A}}-g_{\mathcal{A}^{'}}||^2}{2\sigma_\epsilon^2}\\
&\le \frac{n}{c^2}*c^2 * \frac{2}{2\sigma_\epsilon^2}\\
&=\frac{n}{\sigma^2_\epsilon}
\end{align*}
\endgroup
By the Fano's inequality~\cite{Cover06} on the Markov chain $g\rightarrow S\rightarrow S^{'}\rightarrow \hat{g}$, we have 
\begin{equation*}
\begin{split}
\mathbb{P}[\hat{g} \ne \bar{g}]&\ge 1-\frac{\mathbb{I}(\bar{g},S^{'})+\log 2}{\log c} \ge 1-\frac{n/\sigma^2_\epsilon+\log 2}{\log c}
\end{split}
\end{equation*}
By making \[\frac{1}{2}=\mathbb{P}[\hat{g} \ne \bar{g}]\ge 1-\frac{n/\sigma^2_\epsilon+\log 2}{\log c},\] we have \[n \le (\log c-2\log 2) \sigma^2_\epsilon/2\] Since $c\ge q^{k+1}{p \choose k+1}$, $n \le (\log (q^{k+1}{p \choose k+1})-2\log 2) \sigma^2_\epsilon/2$ implies $\mathbb{P}[\hat{g} \ne \bar{g}]\ge\frac{1}{2}$.
If $p\gg k$, the above is equivalent to 
\begin{equation*}
\begin{split}
n&=\Omega\left(\frac{\sigma^2_\epsilon}{2}(\log [q^{k+1}p^{k+1}/(k+1)!]-2\log 2)\right)\\
&\in \Omega\left((k+1)\log (pq)-\log (k+1)!\right)
\end{split}
\end{equation*} 
\end{proof}

\begin{corollary}
\label{cor2}
Assume nature uniformly picks a true function $\bar{g}$ from $\mathcal{G}_{2k+1}$.  For each $g \in \mathcal{G}_{2k+1}$, define a corresponding $h(\x,y)= \frac{1}{2}(y-g(\x))^2$. The corresponding true hypothesis is $\bar{h}=\bar{h}(\x,y)= \frac{1}{2}(y-\bar{g}(\x))^2$. Let $\mathcal{H}_{2k+1}=\{ h(\x,y)= \frac{1}{2}(y-g(\x))^2, g\in\mathcal{G}_{2k+1}\}$. For any estimator $\hat{h}=\hat{h}(\x,y)= \frac{1}{2}(y-\hat{g}(\x))^2$, if $n\le (\log (q^{k+1}{p \choose k+1})-2\log 2) \sigma^2_\epsilon/2$, then $\E_{\D}[\hat{h}]-E_{\D}[\bar{h}]\ge 1$ with probability at least $\frac{1}{2}$.
\end{corollary}
\begin{proof}
$\bar{g}$ is the true function, so $y=\bar{g}(\x)+\epsilon$, where $\epsilon\sim N(0,\sigma_\epsilon^2)$. Recall that by Theorem 2, if $n\le (\log (q^{k+1}{p \choose k+1})-2\log 2) \sigma^2_\epsilon/2$ then $P[\bar{g} \neq \hat{g}] \geq 1/2$. Thus, assuming that $\bar{g} \neq \hat{g}$, we have
\begin{align*}
\E_{\D}[\hat{h}]-&E_{\D}[\bar{h}]=\frac{1}{2}\E_{(\x,y)\sim \D}[(y-\hat{g}(\x))^2-(y-\bar{g}(\x))^2]\\
=&\frac{1}{2}\E_{\substack{\x\sim Unif[-1,1]^p\\ \epsilon\sim N(0,\sigma^2_\epsilon)}}[(\bar{g}(\x)+\epsilon-\hat{g}(\x))^2-\epsilon^2]\\
=&\frac{1}{2}\E_{\substack{\x\sim Unif[-1,1]^p\\ \epsilon\sim N(0,\sigma^2_\epsilon)}}[(\bar{g}(\x)-\hat{g}(\x))^2+2\epsilon(\bar{g}(\x)-\hat{g}(\x))]\\
=&\frac{1}{2}\E_{\x}[(\bar{g}(\x)-\hat{g}(\x))^2]+\E_{ \epsilon} [\epsilon]*\E_{\x}[(\bar{g}(\x)-\hat{g}(\x))]\\
=&\frac{1}{2}||\bar{g}-\hat{g}||^2\\
=&\frac{1}{2}*2I(\bar{g}\neq\hat{g})\\
=&1
\end{align*}
\end{proof}

\begin{remark}
Excess risk measures how well the empirical risk minimizer performs when compared to the best candidate in the hypothesis class. On the one hand, Corollary \ref{cor1} discusses the upper bound of the excess risk, and indicates that the sufficient sample complexity is $O(k\log (pq)+\log k!)$. On the other hand, Corollary \ref{cor2} discusses the lower bound of the excess risk, and shows that the necessary sample complexity is $\Omega(k\log (pq)-\log k!)$. Especially when $k \ll pq$, both the sufficient sample complexity and necessary sample complexity are $\Theta(k\log (pq))$.
\end{remark}
\section{Greedy Search Algorithm for Regression }

In this section, we propose a greedy search algorithm
to recover a weighted labeled binary tree for regression. 
As mentioned in Section 3.2, for regression, we define $d(y,y') = \min(1, (y-y')^2/2)$. For simplicity, we assume $\mathcal{Y}=[-1,1]$, thus $d(y,y') = (y-y')^2/2$.  Consequently, we have $\mathcal{H}(f)=\{ h(z)= h(x,y)=(y-g(x))^2/2, g\in\mathcal{W}(f)\}$ for a fixed labeled binary tree $f$. The true risk and the empirical risk are defined as $\E_\D[h]=\E_{(x,y)\sim\D}[(y-g(x))^2/2]$, and $\hat{\E}_S[h]=\sum\limits_{i=1}^n(y_i-g(\x_i))^2/2$.

Based on Theorem 1 in Section 3.2, it is straightforward to
have a brute-force algorithm to traverse all possible trees in $\mathcal{F}_{2k+1}$,
and to compute the best weights for each tree. Theorem 1 could guarantee that the risk at the empirical risk minimizer is close to the minimum possible risk over all functions in $\mathcal{W}_{2k+1}$, given enough training samples. However the space of trees grows exponentially with the number of nodes, as shown in Lemma \ref{lemma1}, and therefore the brute-force algorithm is exponential-time.

After decades of work, the literature in tensor decomposition has still failed to provide polynomial-time algorithms with guarantees, for a general nonsymmetric tensor decomposition problem. In general, it has been shown that most tensor problems are NP-hard \cite{hillar2013most}. Therefore most existing literature considers a specific tensor structure like the symmetric orthogonal decomposition \cite{anandkumar2014tensor}. As shown in Figure \ref{treetensor}, we can model the tensor decomposition problem in our framework, for a fixed tree. However in our problem, we learn the tree structure. Thus, our problem is harder than tensor decomposition.

Given the above, we propose a greedy search algorithm for learning the structure of \emph{predictor functions}. A greedy approach was also taken in \cite{structure} for learning the structure of \emph{kernels}. Before we proceed, note that the uniform convergence of the empirical risk to the true risk holds for any $h \in \mathcal{H}_{2k+1}$ and therefore, it applies to the greedy algorithm output, which is an element of $\mathcal{H}_{2k+1}$.

Our algorithm begins by applying all basis functions to all input dimensions, and picking the one that minimizes $\sum_{m=1}^n(y_m-w'\phi_{i'}(x_{j'}))^2/2$ among all function indices $i' \in \{1,\dots,q\}$ and coordinates $j' \in \{1,\dots,p\}$, where $w'$ is estimated separately for each candidate option $(i',j')$. This produces a tree with a single node. After this, we repeat the following search operators over the leaves of the current tree: Any leaf $\mathcal{V}$ can be replaced with $\mathcal{V}+\mathcal{V}^{'}$, or $\mathcal{V}*\mathcal{V}^{'}$, where $\mathcal{V}^{'}=w'\phi_{i'}(x_{j'})$.

Our algorithm searches over the space of trees using a greedy search approach. At each stage, we evaluate the replacement of every leaf by either a summation or multiplication, and compute the weight for the new candidate leaf while fixing all the other weights. Then we take the search operation with the lowest score among all leaves, and adjust all weights by coordinate descent at each iteration.  (For completeness, we include our main algorithm in Appendix C.)

\normalsize

{\bf Computing the Weight.}
A main step in our main algorithm is the computation of the weight of a new candidate leaf, while fixing all the other weights. Fortunately, computing the new weight turns out to be a simple least square problem, but involves traversing the tree from the root to the candidate node being evaluated. (The corresponding algorithm can be found on Appendix C, with a concrete example to illustrate our algorithm.)

{\bf Computational Complexity. }
Next, we analyze the time complexity of our method. In iteration D, we solve $O(pqD)$ single-dimensional closed-form optimization problems: for all the $D$ tree leaves, our algorithm tries to insert a new node with either "+" or "*", all $q$ basis functions, and all $p$ dimensions of $\x$. In addition, it takes $O(nD)$ time to compute the optimal weight (in closed-form) for a specific basis function of a specific dimension of $\x$ at a specific insert position on a dataset of size $n$. Finally, it takes $O(nD)$ to adaptively update all weights at each step by coordinate descent. The computational complexity of our algorithm for $k$ iterations is thus $O(pqn(1^2+2^2+\dots+k^2)) \in O(pqnk^3)$.
This can be reduced by processing the tree leaves (or alternatively, batches of data samples) in parallel.

\section{Experiments}
In this section, we demonstrate our theorem in four simulation experiments. We use a function $g(\x)=0.3sin(3\pi x_1)cos(2\pi x_2)+0.4x_3^2-0.3x_4$, and noise standard deviation $\sigma = 0.05$. Our choice of the set of basis functions $\Phi$ include
B-spline of degree 1, Fourier basis functions: $\{\sin(i\pi x),\cos(i\pi x)\}_{i=1,\dots,\infty}$ and truncated polynomials: $\{x,x^2,x^3,(x-t)_{+}^3,t\in\R\}$, where $(x)_{+}=max(x,0)$. We designed four different experiments to demonstrate our theoretical contributions. For each setting, the generalization error is estimated by the mean of 20 repeated trials in order to show error bars at 95\% confidence level.
\paragraph{Experiment 1}
We set the dimension of the explanatory variables  $p=100$, the number of basis functions: $q=40$, and the number of iterations $k=10$. For each value of $n \in \{ 50, 100, 150, 200, 250\}$, we sampled $n$ random samples $\x_i$,  $y_i=g(\x_i)+\epsilon_i$, $i=1,\cdots,n$ for training, and $n/3$ samples for testing. In Figure \ref{fig:n}, we observe that the generalization error has a sharp decline when $n$ increases from 50 to 100, and a slower decline for higher values of $n$. This demonstrates that the generalization error $\propto \sqrt{\frac{1}{n}}$ as prescribed by Theorem \ref{thm:radunifconv}. 
\begin{figure}[!h]
\centering
	\includegraphics[scale=.35]{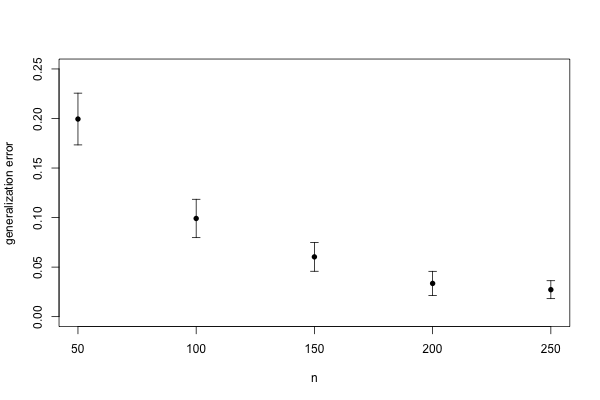}
	\caption{Generalization error vs. sample size $n$.}
	\label{fig:n}
	\end{figure}
\paragraph{Experiment 2}
We set the sample size $n=250$, the number of basis functions: $q=40$, and the number of iterations $k=10$. For each value of $p \in \{ 10, 20, 50, 100, 200\}$, we sampled 250 $p-dimensional$ random samples $\x_i$,  $y_i=g(\x_i)+\epsilon_i$, $i=1,\cdots,n$ for training, and 83 samples for testing. Figure \ref{fig:p} shows that the generalization error grows rapidly when $p\in(0,50)$, and the growth slows down as $p$ increases. This finding matches the conclusion of Theorem \ref{thm:radunifconv} that the generalization error $\propto \sqrt{\log p}$.  
\begin{figure}[!h]
\centering
	\includegraphics[scale=.35]{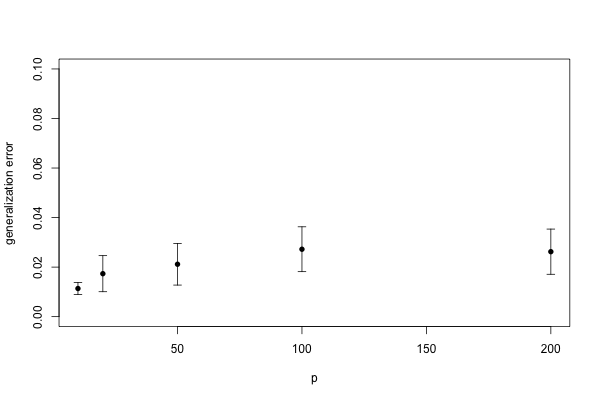}
	\caption{Generalization error vs. dimension of the explanatory variable $p$.}
	\label{fig:p}
\end{figure}
\paragraph{Experiment 3}
We set the dimension of the explanatory variables $p=100$, the number of basis functions: $q=40$, and the sample size $n=250$. For each value of the number of iterations $k \in \{ 1, 5, 10, 20\}$, we sampled 250 random samples $\x_i$,  $y_i=g(\x_i)+\epsilon_i$, $i=1,\cdots,n$ for training, and 83 samples for testing. As shown in Figure \ref{fig:k}, the generalization error grows almost linearly as $k$ increases when $k$ is small, but the growth rate decreases apparently when $k>15$.  This is consistent with the theoretical result that the generalization error $\propto \sqrt{k}$.
\begin{figure}[!h]
\centering
	\includegraphics[scale=.35]{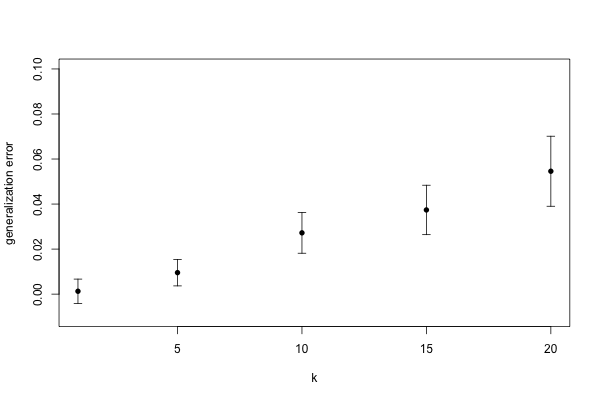}
	\caption{Generalization error vs. number of iterations $k$.}
	\label{fig:k}
\end{figure}
\paragraph{Experiment 4}
We set the dimension of the explanatory variables $p=20$, the sample size $n=250$, and the number of iterations $k=10$. For each value of $q \in \{ 10, 20, 50, 100\}$, we sampled 250 random samples $\x_i$,  $y_i=g(\x_i)+\epsilon_i$, $i=1,\cdots,n$ for training, and 83 samples for testing. Figure \ref{fig:q} indicates that the generalization error grows rapidly when $q$ is small, and the growth slows down as $q$ continue to increase. This matches the conclusion of Theorem \ref{thm:radunifconv} that the generalization error $\propto \sqrt{\log q}$.  
\begin{figure}[!h]
\centering
	\includegraphics[scale=.35]{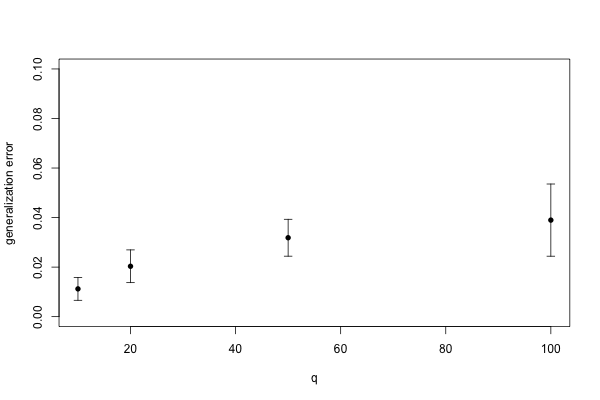}
	\caption{Generalization error vs. number of basis functions $q$.}
	\label{fig:q}
\end{figure}

Our methods are comparative to methods like \emph{Gaussian processes} for two real-world data sets, although our model sizes are much smaller. (Please see Appendix D.)

\section{Concluding Remarks}
There are several ways of extending this research. While we focused on the sample complexity for trees of predictor functions, it would be interesting to analyze trees of kernels as well, as many popular kernel structures \cite{structure} are equivalent to a labeled binary tree. Additionally, while we focused on learning trees, it would be interesting to propose methods for learning general directed acyclic graphs.

\bibliography{references} 

\begin{thebibliography}{10}

\bibitem{anandkumar2014tensor}
Animashree Anandkumar, Rong Ge, Daniel~J Hsu, Sham~M Kakade, and Matus
  Telgarsky.
\newblock Tensor decompositions for learning latent variable models.
\newblock {\em Journal of Machine Learning Research}, 15(1):2773--2832, 2014.

\bibitem{Bartlett02}
P.~Bartlett and S.~Mendelson.
\newblock \uppercase{R}ademacher and \uppercase{G}aussian complexities: Risk
  bounds and structural results.
\newblock {\em Journal of Machine Learning Research}, 3(Nov):463--482, 2002.

\bibitem{spline}
Carl~De Boor.
\newblock {\em A practical guide to splines}, volume~27.
\newblock Springer-Verlag New York, 1978.

\bibitem{cortes2009learning}
Corinna Cortes, Mehryar Mohri, and Afshin Rostamizadeh.
\newblock Learning non-linear combinations of kernels.
\newblock In {\em Advances in Neural Information Processing Systems}, pages
  396--404, 2009.

\bibitem{Cover06}
T.~Cover and J.~Thomas.
\newblock {\em Elements of Information Theory}.
\newblock John Wiley \& Sons, 2nd edition, 2006.

\bibitem{structure}
David~K Duvenaud, James~Robert Lloyd, Roger~B Grosse, Joshua~B Tenenbaum, and
  Zoubin Ghahramani.
\newblock Structure discovery in nonparametric regression through compositional
  kernel search.
\newblock In {\em International Conference on Machine Learning (3)}, pages
  1166--1174, 2013.

\bibitem{hensman2013gaussian}
James Hensman, Nicolo Fusi, and Neil~D Lawrence.
\newblock Gaussian processes for big data.
\newblock {\em Uncertainty in Artificial Intelligence}, 2014.

\bibitem{hillar2013most}
Christopher~J Hillar and Lek-Heng Lim.
\newblock Most tensor problems are {NP-hard}.
\newblock {\em Journal of the ACM (JACM)}, 60(6):45, 2013.

\bibitem{Kakade08}
S.~Kakade, K.~Sridharan, and A.~Tewari.
\newblock On the complexity of linear prediction: Risk bounds, margin bounds,
  and regularization.
\newblock {\em Neural Information Processing Systems}, 21:793--800, 2008.

\bibitem{poon2011sum}
Hoifung Poon and Pedro Domingos.
\newblock Sum-product networks: A new deep architecture.
\newblock In {\em Computer Vision Workshops (ICCV Workshops), 2011 IEEE
  International Conference on}, pages 689--690. IEEE, 2011.

\bibitem{Raskutti09}
Garvesh Raskutti, Bin Yu, and Martin~J Wainwright.
\newblock Lower bounds on minimax rates for nonparametric regression with
  additive sparsity and smoothness.
\newblock In {\em Advances in Neural Information Processing Systems}, pages
  1563--1570, 2009.

\bibitem{Ravikumar07}
Pradeep Ravikumar, Han Liu, John~D Lafferty, and Larry~A Wasserman.
\newblock Spam: Sparse additive models.
\newblock In {\em Neural Information Processing Systems}, pages 1201--1208,
  2007.

\bibitem{santhanam2012information}
Narayana~P Santhanam and Martin~J Wainwright.
\newblock Information-theoretic limits of selecting binary graphical models in
  high dimensions.
\newblock {\em IEEE Transactions on Information Theory}, 58(7):4117--4134,
  2012.

\bibitem{schmidt2009distilling}
Michael Schmidt and Hod Lipson.
\newblock Distilling free-form natural laws from experimental data.
\newblock {\em Science}, 324(5923):81--85, 2009.

\bibitem{Wang10}
W.~Wang, M.~Wainwright, and K.~Ramchandran.
\newblock Information-theoretic bounds on model selection for
  \uppercase{G}aussian \uppercase{M}arkov random fields.
\newblock {\em IEEE International Symposium on Information Theory}, pages 1373
  -- 1377, 2010.

\bibitem{Yu97}
Bin Yu.
\newblock Assouad, {Fano} and {Le Cam}.
\newblock In {\em Festschrift for Lucien Le Cam}, pages 423--435. Springer,
  1997.

\end{thebibliography}
\clearpage
\appendix
\textbf{On the Statistical Efficiency of Compositional Nonparametric Prediction}

\section{Detailed Proofs}
\subsection{Proof for Lemma \ref{lemma0}}
\begin{proof}
We first show $||\uu||_\infty\le 1$: 

For any production of finite basis functions from $\Phi$,
 \[||\prod\limits_{1=1}^L\phi_{i_l}(x_{j_l})||_\infty\le \prod\limits_{1=1}^L||\phi_{i_l}(x_{j_l})||_\infty\le 1\]
Each component of $\uu$ is a production of finite basis functions from $\Phi$. Thus $||\uu||_\infty\le 1$.

Then we show $||\vv||_1\le ||\w||_1$ if $||\w||_1\le 1$ by induction:

$k=0$, $||v||_1=||w||_1$;

Assume that for any $k<K$ and any weighted labeled binary tree $h \in \W_{2k+1}$, $||v_h||_1\le ||w_h||_1$. For $k=K$, decompose the tree $h(\x;f,\w)\in \W_{2K+1}$ by the left subtree $h_l(\x;f_l,\w_l)=\langle \vv_l,\uu_l\rangle$ and the right subtree as $h_r(\x;f_r,\w_r)=\langle \vv_r,\uu_r\rangle$.

If the root is a "+", then $||\vv||_1=||\vv_l||_1+||\vv_r||_1\le ||\w_l||_1 + ||\w_r||_1 = ||\w||_1$.

If the root is a "*", then 
\begin{align*}
||\vv||_1 &= \sum\limits_t\sum\limits_s|v_l^tv_r^s|\\
&=\sum\limits_t|v_l^t|\sum\limits_s|v_r^s|\\
&=\sum\limits_t|v_l^t|||\vv_r||_1\\
&=||\vv_l||_1||\vv_r||_1\\
&\le ||\w_l||_1||\w_r||_1\\
&\le ||\w||_1^2\\
&\le ||\w||_1
\end{align*}
\end{proof}
\subsection{Proof for Lemma \ref{lemma1}}
\begin{proof}
Remind that $p$ is the dimension of the covariate, and $q$ is the number of basis functions. We define $\mathcal{F}^{*}_{2k+1} \subset \mathcal{F}_{2k+1}$ as the set of labeled binary trees with exactly $2k+1$ nodes. In this step, we will show that $|\mathcal{F}^{*}_{2k+1}|\le 2^k(k)!(pq)^{k+1}$.

We first show $|\mathcal{F}^{*}_{2k+1}|\le (pq)^{k+1}(k)!2^k$ for all $k=0,1,\cdots$:

$k=0$, $|\mathcal{F}^{*}_{2*0+1}|=pq \le (pq)^{0+1}(0)!2^0$;

$k=1$, $|\mathcal{F}^{*}_{2*1+1}|=2(pq)^2-2pq < (pq)^{1+1}(1)!2^1$;

Assume that $|\mathcal{F}^{*}_{2*k+1}|\le (pq)^{k+1}(k)!2^k$ for all $k<K$, then for $k=K$,
\begin{align*}
|\mathcal{F}^{*}_{2K+1}|&=2\sum\limits_{i\in \{ 1,3,\cdots,2K-1\} }|\mathcal{F}^{*}_i||\mathcal{F}^{*}_{2K-i}|\\
&\le 2\sum\limits_{i=0,\cdots,K-1} (pq)^{i+1}(i)!2^i\\
&\hspace{.25in}(pq)^{K-i-1+1}(K-i-1)!2^{K-i-1}\\
&=(pq)^{K+1}2^K\sum\limits_{i=0,\cdots,K-1} (i)!(K-i-1)!\\
&\le (pq)^{K+1}2^K\sum\limits_{i=0,\cdots,K-1} (K-1)!\\
&\le (pq)^{K+1}2^K(K)!
\end{align*}
Since for $k\ge 1$, we have $2^{k-1}=\sum\limits_{i=0}^{k-1}\frac{(k-1)!}{i!(k-1-i)!}$, or equivalently, 
$\frac{2^{k-1}}{(k-1)!}=\sum\limits_{i=0}^{k-1}\frac{1}{i!(k-1-i)!}$, 
and since $1/x$ is concave, by Jensen's inequality, we have that $\frac{2^{k-1}}{(k)!}=\sum\limits_{i=0}^{k-1}\frac{1}{k}\frac{1}{i!(k-1-i)!} \le \frac{1}{\sum\limits_{i=0}^{k-1} i!(k-1-i)!/k}$. Thus $\sum\limits_{i=0}^{k-1} i!(k-1-i)! \le k\frac{(k)!}{2^{k-1}}$ for $k \ge 1$.
Except for the root node, a labeled binary tree consists of the left subtree and the right subtree. Thus
\begingroup
\allowdisplaybreaks
\begin{align*}
|\mathcal{F}^{*}_{2k+1}|&=2\sum\limits_{i\in \{ 1,3,\cdots,2k-1\} }|\mathcal{F}^{*}_i||\mathcal{F}^{*}_{2k-i}|\\
&\le(pq)^{k+1}2^k\sum\limits_{i=0,\cdots,k-1} (i)!(k-i-1)!\\
&\le (pq)^{k+1}2^kk\frac{(k)!}{2^{k-1}}\\
&=2k(k)!(pq)^{k+1}
\end{align*}
\endgroup
Finally, we will prove that $|\mathcal{F}_{2k+1}|\le 4k(k)!(pq)^{k+1}$. 
\begin{equation*}
\begin{split}
|\mathcal{F}_{2k+1}|& = \sum\limits_{i=0}^k|\mathcal{F}^{*}_{2i+1}|\\
& \hspace{-.2in}\le \sum\limits_{i=1}^{k-1}2i(i)!(pq)^{i+1}+pq+
2k(k)!(pq)^{k+1}\\
& \hspace{-.2in}\le k*2(k-1)(k-1)!(pq)^{k-1+1}+2k(k)!(pq)^{k+1}\\
& \hspace{-.2in}\le 4k(k)!(pq)^{k+1}
\end{split}
\end{equation*}
\end{proof}
\subsection{Proof for Lemma \ref{lemmadecomposition}}
\begin{proof}
Define $M_{2k+1}^{*}=\max\limits_{f\in \F^{*}_{2k+1}}M_f$. Since $M^{*}_{2k+1}=M_{2k+1}$, it is equivalent to show $M_{2k+1}^{*}<(1.45)^{k+1}$. We will prove the lemma by induction.

$k=0$, $M^{*}_{2*0+1}=1<(1.45)^1$;

$k=1$, $M^{*}_{2*1+1}=max(1,1+1)=2<(1.45)^2$;

$k=2$, $M^{*}_{2*2+1}=3<(1.45)^3$;

Assume that $M_{2k+1}^{*}<(1.45)^{k+1}$ for all $k<K$, where $K\ge 3$, then for $k=K$,
\begin{equation*}
\begin{split}
M_{2k+1}^{*}&=\max\limits_{i\in \{ 1,3,\cdots,2K-1\} }[\max(M^{*}_iM^{*}_{2K-i},M^{*}_i+M^{*}_{2K-i})]\\
&<\max\limits_{i\in \{ 1,3,\cdots,2K-1\} }[\max(1.45^{\frac{i-1}{2}+1}1.45^{\frac{2K-i-1}{2}+1},\\
&\hspace{1.37in}1.45^{\frac{i-1}{2}+1}+1.45^{\frac{2K-i-1}{2}+1})]\\
&=(1.45)^{K+1}
\end{split}
\end{equation*}
\end{proof}
\section{Technical Lemma}
The following technical lemma regarding the McDiarmid's condition for the supremum can be found in \cite{Bartlett02}.
\begin{lemma} \label{lem:varphilipschitz}
Let $z$ be a random variable of support $\Z=(\mathbb{R}^p,\mathcal{Y})$ and distribution $\D$.
Let ${S = \{z_1 \dots z_n\}}$ be a dataset of $n$ samples.
Let $\HH$ be a hypothesis class satisfying ${\HH \subseteq \{h \mid h : \Z \to [0,1]\}}$.
The function:
\begin{align} \label{eq:varphi}
\varphi(S) = \sup_{h \in \HH}{\left( \E_\D[h] - \Eh_S[h] \right)}
\end{align}
\noindent satisfies the following condition:
\begin{align*}
{\rm\ } |\varphi(z_1,\dots,z_i,\dots,z_n) - \varphi(z_1,\dots,\zp_i,\dots,z_n)| \leq 1/n\\
(\forall i, \forall z_1 \dots z_n, \zp_i \in \Z)
\end{align*}
\end{lemma}
\section{Detailed Greedy Search Algorithm and Illustration Example}
For completeness, we present our main greedy search algorithm in detail in Algorithm \ref{alg:main}, as well as the algorithm to compute the node weights in Algorithm \ref{alg:example}.
For simplicity, we assume the covariate $\bm{x}_m\in [0,1]^p$. As for the set of basis functions $\Phi$, piecewise linear functions, Fourier basis functions, or truncated polynomials   could be good choices in practice. We first define $f_{\w}(\x)$ as the output of tree structure $f$ with weights $\w$ for input $\x$. For instance, let $f$ be the tree structure of Figure \ref{tree1}. With a corresponding weight for each leaf, $f_{\w}$ can be visualized as in Figure \ref{wtree}. Thus $f_{\w}(\x)=(w_1\phi_1(x_2)+w_2\phi_3(x_1))*(w_3\phi_3(x_2)+w_4\phi_1(x_3))$ in this specific case. The loss function is defined as $L(f_{\w};\x,\bm{y})=\sum\limits_{m=1}^n(y_m-\hat{y}_m)^2/2$, where $\hat{y}_m=f_{\w}(\x_m)$. We could explore the interaction structure $f$ by adding and multiplying a basis function on a single dimension of covariate $\bm{x}$. 
\begin{algorithm}
   \caption{Greedy search algorithm}
   \label{alg:main}
\begin{algorithmic}

   \STATE {\bfseries Input:} $\bm{X}=(\x_1,\dots,\x_n)'\in \R^{n\times p}$: n data points
   
    \hspace{.44in}$\bm{y}=(y_1,\dots,y_n)\in \R^n$: n observations

   \hspace{.44in}$\Phi$: a set of $q$ basis functions, $k$: the number of iterations
   
   \STATE Initialize the tree  $f_{\w}=w_0+w_1\phi_{i_1}(x_{j_1})$, where
   
$(w_0,w_1,i_1,j_1)=\argmin\limits_{(w_0',w',i',j')}\sum_{m=1}^n(y_m-w_0'-w'\phi_{i'}(x_{j'}))^2$
    
   \FOR{$iters=1$ {\bfseries to} $k-1$}
    \FOR{$node$ in $f_{\w}.leaves$}
     \STATE $path=path(f_{\w}.root,node)$
     \FOR{$m=1$ {\bfseries to} $n$}
     \STATE Algorithm \ref{alg:example} with input 
     $(\x_m,f_{\w},path)$:
     
      $b_m=b(\x_m)$, $k_m=k(\x_m)$
     \STATE $c_m=node(\x_m)$ (If $node$ is $w\phi_i(x_j)$, then $node(\x_m)=w\phi_i(x_{mj})$)
     \ENDFOR
     
    \STATE 
    $(w_0,w_+,i_+,j_+)=\argmin\limits_{(w_0',w'\le 1,i',j')}\sum\limits_{m=1}^n(y_m-w_0'-b_m-k_m(c_m+w'\phi_{i'}(x_{mj'})))^2$, and define $r_+$ as the corresponding minimum value attained. 

    $(w_0,w_*,i_*,j_*)=\argmin\limits_{(w_0',w'\le 1,i',j')}\sum\limits_{m=1}^n(y_m-w_0'-b_m-k_m(c_mw'\phi_{i'}(x_{mj'})))^2$, and define $r_*$ as the corresponding minimum value attained.  
    
    \IF{$r_+<r_*$}
   \STATE Insert the new leaf $w_+\phi_{i_+}(x_{j_+})$ at $node$ with "+", and call the new tree $f_{\w}^{node}$
   \STATE $r_{node}=r_+$
    \ELSE
    \STATE Insert the new leaf$(w_*\phi_{i_*}(x_{j_*})$ at $node)$ with "*", and call the new tree $f_{\w}^{node}$
    \STATE $r_{node}=r_*$
    \ENDIF
    \STATE Adjust all weights
    \ENDFOR
    \IF{$r_{node}<r_{best}$}
    \STATE $r_{BEST}=r_{node}$, $f_{\w}^{best}=f_{\w}^{node}$
    \ENDIF
   \STATE Update $f_{\w}$ with $f_{\w}^{best}$
   \ENDFOR
    \STATE {\bfseries Output:} $f_{\w}$ 
\end{algorithmic}
\end{algorithm}

  \begin{algorithm}
   \caption{Compute $b(\x_m)$ and $k(\x_m)$}
   \label{alg:example}
\begin{algorithmic}
   \STATE {\bfseries Input:} $\x_m\in \R^p$: data point
   
   \hspace{.44in}$f_{\w}$: current weighted labeled tree

   \hspace{.44in}$path$: path from the root to the insert position
   
   \STATE Initialize $root$ as the root of $f_{\w}$, k=1, b=0
   \WHILE{$path$ is not empty}
   \STATE Define $subtree$ as the $!path[1]$ subtree of $root$
   \STATE $val=evaluate(subtree,\x_m)$, where $evaluate$ gives the output of the weighted labeled tree $subtree$ with input $\x_m$
   \IF{$root="+"$} 
   \STATE $b=b+val*k$
   \ELSIF{$root="*"$} \STATE $k=val*k$ 
   \ENDIF \STATE Update $root$ as its $path[1]$ child
   \STATE Remove the first element of $path$
   \ENDWHILE
   \STATE {\bfseries Output:} $(b, k)$ 
\end{algorithmic}
\end{algorithm}

\paragraph{An example to illustrate Algorithm \ref{alg:example}.}
Take Figure \ref{algorithmtree} for example, and assume we are trying to insert a new leaf $wx_4^3$ with either a "+" or "*" at the Node E, that is to replace the weighted leaf $-.05x_1$ with either $-.05x_1+wx_4^3$ or $-.05x_1*wx_4^3$. With an unknown weight $w$ and an unknown intercept $w_0$, the output $\hat{y}_m$ for the input $\x_m$ of the new tree is 
\begin{align*}
&w_0+[.1x_{m2}^2-.05x_{m1}+wx_{m4}^3](.3\sin(\pi x_{m2})+.02x_{m3})\\
&\triangleq w_0+b(\x_m)+k(\x_m)(wx_{m4}^3-.05x_{m1})
\end{align*}
for "+", and 
\begin{align*}
&w_0+[.1x_{m2}^2+wx_{m4}^3(-.05)x_{m1}](.3\sin(\pi x_{m2})+.02x_{m3})\\
&=w_0+b(\x_m)+k(\x_m)(-.05wx_{m4}^3x_{m1})
\end{align*}
 for "*". 
 \begin{figure}[!h]
\begin{center}
\centerline{\includegraphics[scale=.4]{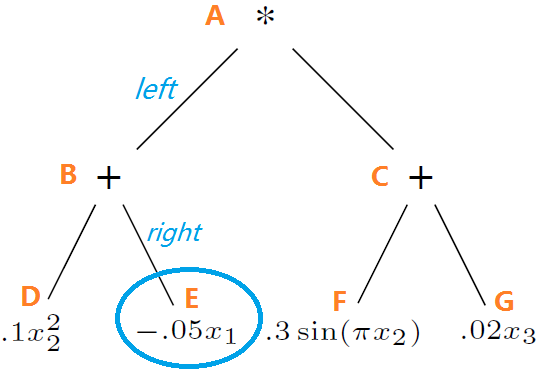}}
\caption{Inserting a new leaf at Node E.}
\label{algorithmtree}
\end{center}
\end{figure} 
 
 Note that $b(\x_m)$ and $k(\x_m)$ are constant with respect to the to-be-defined weight, and thus, the optimization problems $\min\limits_w {\sum\limits_{m=1}^n(y_m-w_0-b(\x_m)-k(\x_m)(wx_{m4}^3-.05x_{m1}))^2}$ and $\min\limits_w {\sum\limits_{m=1}^n(y_m-w_0-b(\x_m)-k(\x_m)(-.05wx_{m4}^3x_{m1}))^2}$ are both least square problems. We add a constraint $|w|\le 1$ according to the assumption of Theorem \ref{thm:radunifconv}, to ensure the uniform convergence. However, it is not straightforward to compute $b(\x_m)$ and $k(\x_m)$. As shown in Algorithm \ref{alg:example}, we compute the value of $b(\x_m)$ and $k(\x_m)$ iteratively along the path from the root to the insert position.
 We continue with our current setting, and move on to compute $b(\x_m)$ and $k(\x_m)$ according to Algorithm \ref{alg:example}, assuming $\x_m=(1,1,1)$. 
 \begingroup
\allowdisplaybreaks
 \begin{enumerate}
 \item Input: $\x_m=(1,1,1)$, $f_{\w}$ is the tree in Figure \ref{algorithmtree}, $path=(left,right)$
 \item Initialize: $root=$Node A, $k=1$,$b=0$
 \item In a first iteration $path[1]=left$, so define $subtree$ as the $right=!left$ subtree of $root$(consisting of Nodes C, F, G), 
 
 $val_m=evaluate(subtree,\x_m)=.3\sin (\pi x_{m2})+.02x_{m3}=.02$
 \item Since $root="*"$, $k=val_m*k=.02$
 \item Update $root$ as its left child: $root=$Node B, $path=(right)$ after removing the first element of path 
 \item In a second iteration $path[1]=right$, so update $subtree$ as the $left=!right$ subtree of $root$ (consisting of Node D only)
 
 $val_m=evaluate(subtree,\x_m)=.1x_{m2}^2=.1$
 \item Since $root="+"$, $b=b+val_m*k=.002$
 \item Update $root$ as its right child, $path=()$ after removing the first element of path
 \item Stop the iterations since $path$ is empty
 \item Return $(b(\x_m)=.002,k(\x_m)=.02)$  
 \end{enumerate}
\endgroup
\section{Real World Experiments}
\paragraph{Airline Delays.}
For real-world experiments, we evaluate our algorithm on the US flight dataset. We use a subset of the data with flight arrival and departure times for commercial flights in 2008. The flight delay is the response variable, which is predicted by using the following variables: the age of the aircraft, distance that needs to be covered, airtime, departure time, arrival time, day of the week, day of the month, and month. We randomly select 800,000 datapoints, using a random subset of 700,000 samples to train the model and 100,000 to test it. Although our method uses only $k=10$ (i.e., $2k+1=21$ nodes, or $k+1=11$ functions of features), we obtain a test RMSE of 34.89. For comparison, the authors in \cite{hensman2013gaussian} also randomly selected 800,000 samples (700,000 for training, 100,000 for testing) and obtained an RMSE between 32.6 and 33.5 with 1200 iterations on a \emph{Gaussian processes} approach.
In general, Gaussian processes predict the output by memorization of the 700,000 training points. Our tree depends only on evaluating $k+1=11$ functions of features. When predicting, our tree does not need to remember the training set.

\paragraph{World Weather.}
The world weather dataset contains monthly measurements of temperature, precipitation, vapor, cloud cover, wet days and frost days from Jan 1990 to Dec 2002 (156 months) on a $5\times 5$ degree grid that covers the entire world. The dataset is publicly available at http://www.cru.uea.ac.uk/. The response variable is temperature. We use 19,000 samples for training, 8000 samples for testing, and run 30 iterations. Although our method uses only $k=30$ (i.e., $2k+1=61$ nodes, or $k+1=31$ functions of features), we obtain a test RMSE of 1.319. Gaussian processes obtained a test RMSE of 1.23. Since the standard deviation of the output variable is 16.98, both our method and Gaussian processes obtain a coefficient of determination of 0.99.
\end{document}